\documentclass{article} 
\usepackage{iclr2026_conference,times}

\usepackage{amsmath,amsfonts,bm}

\def\eqref#1{equation~\ref{#1}}

\def\1{\bm{1}}

\DeclareMathAlphabet{\mathsfit}{\encodingdefault}{\sfdefault}{m}{sl}
\SetMathAlphabet{\mathsfit}{bold}{\encodingdefault}{\sfdefault}{bx}{n}

\usepackage{graphicx}
\usepackage{amsmath,amssymb,amsthm}
\usepackage[utf8]{inputenc} 
\usepackage[T1]{fontenc}    
\usepackage{hyperref}       
\usepackage{url}            
\usepackage{booktabs}       
\usepackage{amsfonts}       
\usepackage{nicefrac}       
\usepackage{microtype}      
\usepackage{xcolor}         
\usepackage[table]{xcolor}
\usepackage{subcaption}
\definecolor{SecondFill}{HTML}{DDEBFF} 
\definecolor{ThirdFill}{HTML}{FFF0D6}  
\newcommand{\second}[1]{\cellcolor{SecondFill}{#1}}        
\newcommand{\third}[1]{\cellcolor{ThirdFill}{#1}}

\newcommand{\passone}{\text{Pass@1}}
\newcommand{\passk}{\text{Pass@}k}
\newcommand{\covtau}{\text{Cover@$\tau$}}
\hypersetup{unicode=true}
\DeclareRobustCommand{\covtauunicode}{%
  \texorpdfstring{Cover@\(\tau\)}{Cover@τ}
}
\newtheorem{definition}{Definition}

\newtheorem{proposition}{Proposition}
\newtheorem{corollary}{Corollary}
\newtheorem{remark}{Remark}

\title{Beyond $\passk$: Breadth-Depth Metrics
for Reasoning Boundaries}

\author{Marius Dragoi, Ioana Pintilie, Florin Gogianu \& Florin Brad \\
Bitdefender, Romania\\
\texttt{\{mdragoi,ipintilie,fgogianu,fbrad\}@bitdefender.com} \\
}

\iclrfinalcopy 
\begin{document}

\maketitle

\begin{abstract}
  Reinforcement Learning with Verifiable Rewards (RLVR) has emerged as a powerful paradigm to improve Large Language Models on reasoning tasks such as coding, math or logic. To assess the reasoning boundary (the fraction of problems a model can solve) researchers often report $\passk$ at large sampling budgets. Recent results reveal a crossover phenomenon: while RLVR models outperform the base model at small k values, the base model usually outperforms them when sampling a very large number of completions. This has been interpreted as evidence that base models have a larger reasoning boundary. We argue that on tasks with discrete answer spaces, such as math with numeric outputs, $\passk$ at large k reflects the increasingly higher chance of success in the limit of the number of trials rather than genuine reasoning, and can therefore be misleading. We propose $\covtau$, which measures the fraction of problems that a model can solve for which at least a $\tau$ proportion of completions are correct. Unlike $\passk$, $\covtau$ captures reasoning under an explicit reliability threshold: models that rely on random guessing degrade rapidly as $\tau$ increases. We evaluate several RLVR models using $\covtau$-based metrics and illustrate how the relative rankings of popular algorithms change compared to $\passone$, offering a different perspective on reasoning boundaries.
\end{abstract}

\section{Introduction}

 Reinforcement Learning with Verifiable Rewards \citep{guo2025deepseek} has become an essential post-training approach for improving the capability of LLMs on math, code and logical reasoning. Recent research \citep{wang2025reinforcement,wu2025invisible} has called into question the extent to which the RLVR models truly expand the reasoning boundary (i.e., extending the scope of solvable tasks). These works report $\passk$ for increasing values of k and the reasoning boundary is defined as $\passk$ at large k. While the RLVR model has higher Pass@1, at large k, the base model eventually outperforms it, thus the reasoning boundary of the base model shrinks as a result of applying RLVR. Follow-up works use this crossover plot to illustrate whether RLVR expands the reasoning boundary, showing that it can expand or shrink depending on the domain \citep{liu2025prorl,cheng2025revisiting}.

 For tasks with numerical answers, such as math, $\passk$ at large k may eventually produce correct answers for all problems, due to random guessing rather than reasoning ability. We thus argue that $\passk$ can be problematic as a measure for the reasoning boundary, because it doesn't account for any level of reliability, for any given problem. We thus complement it with a \textit{reliability-controlled reasoning boundary}, that exposes the reliability level explicitly. Concretely, we define $\covtau$, which measures the fraction of problems solved by a model with success rate at least $\tau$. For very small $\tau$, $\covtau$ behaves similarly to $\passk$, however, increasing $\tau$ tightens the criterion to emphasize consistency over chance. Figure~\ref{fig:pass_at_k_example} highlights the behavior of $\passk$ and $\covtau$ for a math dataset from OMEGA \citep{sun2025omega} with numerical answers and small set support.

 Our contributions are as follows:
\begin{itemize}
    \item we introduce $\covtau$, a reliability-thresholded metric that reveals the reasoning abilities under a different reliability level $\tau$. $\covtau$ offers a more informative view, highlighting an \textit{explicit breadth-depth trade-off}: low $\tau$ captures breadth of problem solving (even if unreliable), while larger $\tau$ captures depth (problem solving with high reliability). Measuring $\covtau$ reveals a different ranking of popular RLVR algorithms when compared to Pass@1 or $\passk$ at large k; this shows a complementary perspective of the model capabilities
    \item we demonstrate that $\passk$ is a weighted average of $\covtau$, with weights from a Beta(1,k) distribution; this reveals that $\passk$ is biased towards low-$\tau$ regions of $\covtau$, emphasizing lucky hits rather than reliability
    \item we illustrate the usefulness of $\covtau$ by characterizing the performance of different RLVR methods trained on math datasets from OMEGA and Reasoning Gym.
\end{itemize}

\begin{figure}[t]
  \centering
\includegraphics[width=\linewidth]{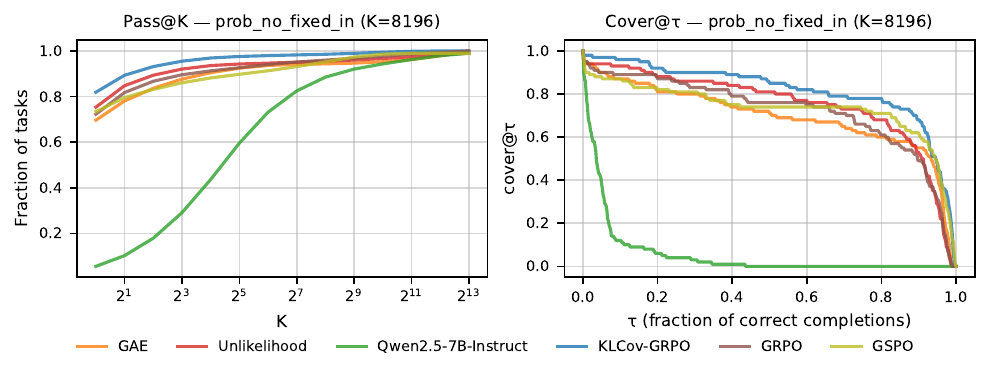}
  \caption{$\passk$ and $\covtau$ curves for Qwen2.5-7B-Instruct and several RLVR models on the Probability set of OMEGA. \textbf{Left:} $\passk$ quickly saturates for larger $k$ due to small test support. \textbf{Right:} $\covtau$ illustrates a more gradual assessment of the models' capabilities, ranging from maximum performance (at low $\tau$ values) to very limited capabilities (when requiring models have almost perfect reliability at high $\tau$ values).}
  \label{fig:pass_at_k_example}
\end{figure}

\section{Related work}
DeepSeekR1 \citep{guo2025deepseek} has paved the way for reasoning-focused LLMs, by finetuning LLMs with reinforcement learning from verifiable rewards. The model employs the Group Relative Policy Optimization (GRPO) algorithm \citep{shao2024deepseekmath}, which is a critic-free variant of PPO \citep{schulman2017proximal}, that estimates the baseline from the group average rewards. 

\subsection{Evaluation metrics}
Models are usually evaluated via $\passk$ \citep{chen2021evaluating}, using $k=1$ as the most common value. Higher values of $k$ (e.g. $16$ or $64$) are also used to test the upper bound capabilities of the models. 
Recent work argues that the relative decrease in reasoning capacity of RLVR models compared to the base model is an artefact of $\passk$ and proposes CoT-Pass@k \citep{wen2025reinforcement} to account for the correctness of both the thinking tokens and the final answer.
Finally, some papers report maj@k or cons@k \citep{guo2025deepseek,shao2024deepseekmath}, which counts a problem as solved based on aggregation over 
k samples, either majority (>50\% of completions correct) or mode (most frequent). Both metrics target whether a model is consistent and are related to our proposal. Maj@k is equivalent to $\covtau$ with $\tau=0.5$. Cons@k has no fixed reliability threshold: the effective $\tau$
varies per problem with the mode’s frequency, so it does not correspond to a single $\covtau$. Our $\covtau$ is structurally similar to the performance profiles advocated by \cite{agarwal2021deep}, where the fraction of solved games is reported relative to varying human-normalized score levels. The key distinction is we evaluate coverage under varying reliability levels using a fixed (and sufficiently large) sample budget of $k$ completions, without normalizing to human performance.

\cite{liu2024your} et al. introduce G-Pass$@k_\tau$, a generalization of $\passk$ that evaluates the capability of the model under different consistency levels $\tau$. Our metric $\covtau$ was developed independently and, while conceptually similar, differs in formulation and intended use. First, our metric assumes that tasks have different per-trial success probabilities $p_i$, estimated directly from the n completions of each task, eliminating the need to set a separate k parameter. Second, we introduce our metric in the context of the reasoning boundary debate, whereas \cite{liu2024your} use G-Pass$@k_\tau$ primarily to assess model-level stability under repeated sampling. To this end, G-Pass$@k_\tau$ focuses on higher reliability levels ($\tau \in [0.5,1.0]$), while we use $\covtau$ to analyze the breadth-depth trade-off, by evaluating performance at both low ($\tau=0.2$) and high ($\tau=0.8$) reliability thresholds.

\subsection{Exploration-promoting methods}
Several shortcomings in GRPO have been reported, notably optimization biases \citep{liu2025understanding,zheng2025group} and entropy collapse \citep{yu2025dapo}. The latter is particularly problematic, because it reduces exploration and quickly saturates performance \citep{cui2025entropy}. Follow-up work addresses the entropy collapse and encourages exploration. Simpler approaches add an explicit entropy loss term \citep{wang2025reinforcement} or increase the upper clipping threshold to favor low-probability exploration tokens \citep{yu2025dapo}. More fine-grained techniques such as KL-cov \citep{cui2025entropy} suppress the high-covariance tokens that correlate with large decreases in entropy. Other works identify forking tokens \citep{wang2025beyond,cheng2025reasoning}, which are high entropy tokens serving as logical connectors. Restricting policy gradient updates to these high-entropy tokens maintains entropy and enhances exploration. 

Another line of work boosts exploration and {\passk} performance by optimizing for {\passk} directly \citep{chen2025pass} or promoting low-probability solutions \citep{he2025rewarding}.

\subsection{Procedurally generated datasets}
Progress of RLVR methods on math reasoning has been questioned, due to potential contamination in popular static benchmarks. Portion of these datasets may appear in the pretraining data of popular LLMs \citep{wu2025reasoning}, which can conflate reasoning abilities with memorization. To limit contamination concerns, recent math and logic datasets are procedurally generated and designed with increasing difficulty and structural variation \citep{sun2025omega,stojanovski2025reasoning}. 

\section{{\passk} can be misleading}

We consider $T$ tasks, where task $i$ has per-trial success probability $p_i \in [0,1]$ under i.i.d.\ trials. $\passk$ is defined as the average probability that a task is solved within k independent attempts.
\begin{definition}[$\passk$]\label{def:pak}
\[
\mathrm{\passk} \;=\; \frac{1}{T}\sum_{i=1}^T \Big(1 - (1-p_i)^k\Big).
\]
\end{definition}

$\passk$ has been used to assess the reasoning boundaries of RLVR models by comparing the performance of a finetuned model against the base model at varying values of $k$. At large $k$ values, the performance of the base model typically approaches or even surpasses that of the finetuned model, seemingly closing any gap that may exist at small $k$ values. This pattern, however, does not necessarily reflect genuine reasoning ability. Instead, it primarily highlights the diversity of output trajectories in the base model. In the mathematical reasoning setting, where tasks often involve numerical answers with very limited support, this effect can create a misleading impression about model reasoning.

Figure \ref{fig:pass_at_k_example} illustrates this point by plotting $\passk$ curves for the base model and several RLVR variants on the Probability (No Fixed) task from the OMEGA dataset. Given enough trials ($2^{13}$), the base model achieves $\passk = 1$, due to the very limited size of the output space (only 30 possible values in the test set). More generally, this phenomenon is inevitable: provided there is nonzero probability of producing the correct answer, $\passk$ will always converge to 1 in the limit of infinite trials.

\begin{remark}[Degeneracy of $\passk$ at Large $k$]\label{remark:degen}
For any success probability $0 < p \leq 1$,
\[
\lim_{k \to \infty} \big(1 - (1-p)^k\big) = 1.
\]
\end{remark}

We thus argue that using $\passk$ as a proxy for reasoning boundaries can be misleading, because it confounds true capability with random chance.

\section{{\covtauunicode}}
Whereas $\passk$ captures the binary likelihood of success within $k$ attempts, we propose examining how dataset coverage changes when accounting for the consistency of predictions. We define $\covtau$ as fraction of problems for which at least a proportion $\tau$ of the generated completions are correct. Formally:

\begin{definition}[$\covtau$]\label{def:cov}
For any threshold $\tau \in [0,1]$, define
\[
G(\tau) \;=\; \frac{1}{T}\sum_{i=1}^T \mathbf{1}\{p_i \geq \tau\},
\]
i.e. the fraction of tasks that can be solved with per-trial success probability at least $\tau$.
\end{definition}

Consider two LLMs, tested on the same set of problems: A has probability of success of 0.5 on all problems, while B has probability of success 0 on half the problems and 1 on the other half. Both models have $Pass@1 = 0.5$, but the \textbf{$\covtau$} plot in the right side of Figure \ref{fig:cover_toy_example} also shows model performances at explicit reliability levels: model A solves more problems, while model B solves fewer problems, but more consistently. Thus, $\covtau$ captures a fine-grained view of the performance regime of the models:  
\begin{itemize}
\item low $\tau$ values: $\covtau$ highlights the \textbf{breadth} of the capabilities (how many problems are at least sometimes solvable)
\item higher $\tau$ values: $\covtau$ indicates \textbf{depth} (how reliably a problem is solved)
\end{itemize}

\begin{figure}[t]
  \centering
  \includegraphics[width=1.0\linewidth]{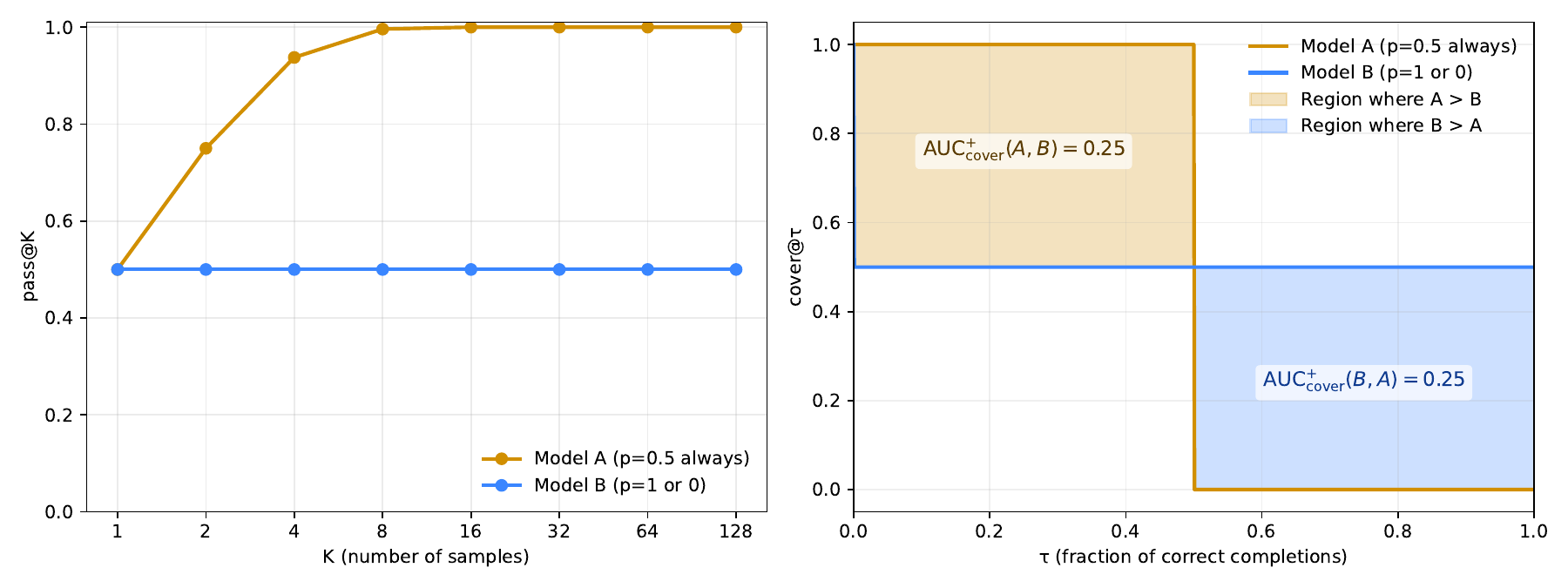}
  \caption{\textbf{Left:} pass@K for two models A and B. Both have the same pass@1=0.5, but model A's reasoning boundary increases with more tries, while model B stay flat. \textbf{Right:} $\covtau$ curves for the same models A and B. Model A solves more problems overall, while Model B solves fewer problems but with higher consistency. When comparing their excess AUC (areas where each curve dominates), their overall advantages balance out.}
  \label{fig:cover_toy_example}
\end{figure}

Two models A and B can cross over in their $\covtau$ curves: model A performs better at lower reliability thresholds, but B dominates at higher thresholds. To capture such cases where models trade dominance at different $\tau$ levels, we define $\text{AUC}^{+}_{cover}(A,B)$. This is the excess AUC between model's A coverage over model's B coverage across reliability thresholds:
$$\text{AUC}^{+}_{cover}(A,B) = \int_{0}^{1} \max\!\big(G_A(\tau) - G_B(\tau), \, 0\big)\, d\tau,
$$ where $G_M(\tau)$ denotes the $\covtau$ curve of model $M$.
This pairwise metric captures the total coverage advantage of model A relative to model B, ignoring regions where coverage of A is worse. In the right side of Figure~\ref{fig:cover_toy_example}, both models have equal $\text{AUC}^{+}_{cover}$, indicating a similar level of performance.

\section{Relating {\passk} and {\covtauunicode}} \label{theory}

We now demonstrate the connection between {\passk} and {\covtau}. Specifically, $\passk$ is a Beta-weighted average of the $\covtau$ metric, and thus represents only one particular projection of the richer information contained in $\covtau$.
\begin{proposition}[]\label{prop:pak_as_weighted}
$\passk$ as Weighted Average of $\covtau$. 

For any $k \geq 1$,
\[
\mathrm{\passk}
= \int_0^1 k(1-\tau)^{k-1}\, G(\tau)\, d\tau,
\]
where $G(\tau)$ denotes the $\covtau$ curve.
\end{proposition}

\begin{proof}
For a single task with success probability $p$, define $f_k(p)=1-(1-p)^k$.
Since $f'_k(p)=k(1-p)^{k-1}$ and $f_k(0)=0$, the fundamental theorem of calculus gives
\[
f_k(p) = \int_0^p f'_k(\tau)\, d\tau
= \int_0^1 k(1-\tau)^{k-1}\, \mathbf{1}\{p \ge \tau\}\, d\tau.
\]
Averaging over tasks,
\[
\passk = \frac{1}{T}\sum_{i=1}^T f_k(p_i)
= \int_0^1 k(1-\tau)^{k-1}\,
\Big(\tfrac{1}{T}\sum_{i=1}^T \mathbf{1}\{p_i \ge \tau\}\Big)\, d\tau.
\]
The term in parentheses is exactly $G(\tau)$, which yields the claim.
\end{proof}

\begin{corollary}[]\label{cor:pak_as_expectation_under_beta}
$\passk$ as Expectation under Beta Weights.

Let $\tau \sim \mathrm{Beta}(1,k)$ with density 
$p_k(\tau)=k(1-\tau)^{k-1}$ on $[0,1]$.
Then
\[
\passk
=\int_0^1 G(\tau)\, p_k(\tau)\, d\tau
=\mathbb{E}_{\tau\sim \mathrm{Beta}(1,k)}[\,G(\tau)\,].
\]
\end{corollary}

\begin{corollary}[]\label{cor:auc_eq_pass1}
Uniform AUC Equals Pass@1.

The unweighted area under the $\covtau$ curve satisfies
\[
\int_0^1 G(\tau)\, d\tau = \frac{1}{T}\sum_{i=1}^T p_i = \mathrm{Pass@1}.
\]
\end{corollary}

\begin{remark}[]\label{remark:asymp}
As $k \to \infty$, the weighting distribution $\mathrm{Beta}(1,k)$ concentrates at $\tau=0$. Therefore
\[
\lim_{k \to \infty} \passk = G(0^+),
\]
the fraction of tasks with nonzero success probability. If every $p_i > 0$, then $\passk \to 1$.
\end{remark}

\begin{proposition}[]\label{prop:cover_includes_pass}
{\covtau} dominance implies {\passk} dominance.

Given two models A and B, if $\covtau(A) \ge \covtau(B)$ for all $\tau \in [0,1]$, then
$\passk(A) \ge \passk(B)$ for every $K \ge 1$.

\end{proposition}

\begin{proof}
Using proposition 1,
\[
\mathrm{\passk}(A)-\mathrm{\passk}(B)
= \int_0^1 k(1-\tau)^{k-1}\, (G_A(\tau)-G_B(\tau))\, d\tau,
\]
Since $k(1-\tau)^{k-1} \ge 0$ and $G_A(\tau)-G_B(\tau) \ge 0$, the integral is positive. Therefore $\mathrm{\passk}(A) \ge \mathrm{\passk}(B)$.
\end{proof}

\section*{Observations}

\paragraph{$\passk$ as a weighted summary of $\covtau$.}
Proposition~1 shows that $\passk$ is a weighted average of the $\covtau$ curve,
where the weights are given by a $\mathrm{Beta}(1,k)$ distribution.
In other words, $\passk$ summarizes only part of the information already contained in $\covtau$.

\paragraph{Bias toward low thresholds.}
The $\mathrm{Beta}(1,k)$ weighting heavily emphasizes small values of $\tau$ as $k$ grows.
Thus $\passk$ primarily captures whether tasks have \emph{any} nonzero success probability,
rather than how reliably they can be solved.
In the limit $k \to \infty$, $\passk$ collapses to the trivial statistic
“fraction of tasks with $p_i>0$.”

\paragraph{Uniform weighting recovers $\passone$.}
When $G(\tau)$ is weighted uniformly, the resulting area under the $\covtau$ curve equals $\passone$,
i.e.\ the average per-trial success probability.
This highlights that $\covtau$ naturally generalizes both Pass@1 and $\passk$.

\paragraph{{\covtau} ranking is more informative than {\passk}}
Proposition~2 shows that rankings based on {\covtau} imply the corresponding rankings based on {\passk}. However, the reciprocal is not true, as evidenced by Figure~\ref{fig:cover_toy_example}. Model A outperforms model B in terms of {\passk} for all K, but the {\covtau} curve reveals ordering differences across reliability levels that the {\passk} curve can hide.

\medskip
\noindent\textbf{Summary.}
$\covtau$ makes the coverage-reliability trade-offs explicit, avoids the degeneracy of large-$k$ behavior,
and reveals finer-grained rankings that {\passk} can obscure.

\section{Re-assesing generalization performance on math using Cover@tau}

\begin{figure}[h]
  \centering
  \includegraphics[width=1.\linewidth]{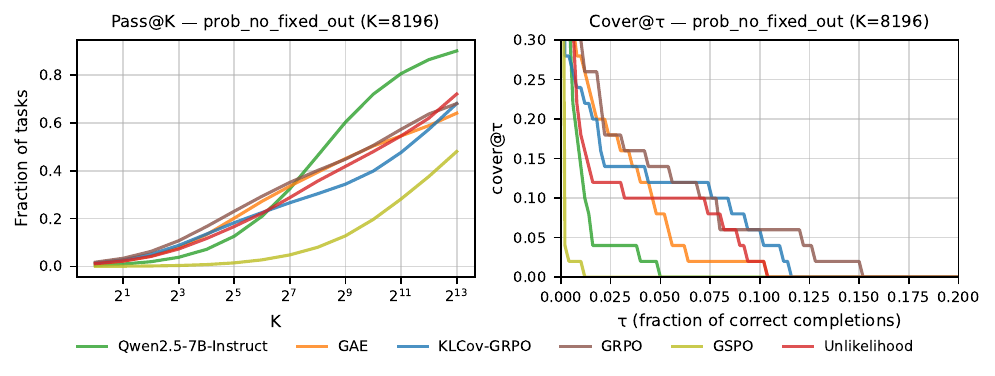}
  \caption{{\passk} and {\covtau} curves for Qwen2.5-7B-Instruct and RLVR models on the Probability (No Fixed) subset of OMEGA, for the OOD test split. \textbf{Left:} All models have poor accuracy, and increasing the sampling budget leads to higher $\passk$, especially on the base model. \textbf{Right:} GRPO and KL-Cov generalize the best; the base model quickly drops in performance even at low reliability thresholds, suggesting a far more limited reasoning boundary than the $Pass@k$ plot implies.}
  \label{fig:cover_tau_probs}
\end{figure}

We evaluate RLVR methods using Cover@{$\tau$}-based metrics, focusing on: (i) highlighting the trade-offs exposed by different $\tau$ values (ii) examining how models perform in mathematical reasoning under OOD settings.

\subsection{Experimental setup}

\paragraph{Datasets}Popular datasets for evaluating mathematical reasoning may suffer from data leakage \citep{wu2025reasoning}. As a result, training on these datasets may not accurately reflect improvements relative to the base model. We instead focus on two particular datasets: OMEGA \citep{sun2025omega} and Reasoning Gym \citep{stojanovski2025reasoning}. OMEGA provides both IID and OOD test splits across various tasks, from which we select: GCD, Function Intersection, Probability (No Fixed), Digit Sum, and Circle. For all tasks, we train on IID data and test on the OOD splits. For Reasoning Gym, we follow the intra-domain transfer setup: we train on two tasks from the algebra domain (simple equations and polynomial multiplication) and test on a third (intermediate integration).

\paragraph{Methods} We evaluate two commonly used RLVR algorithms: GRPO \citep{shao2024deepseekmath} and PPO \citep{schulman2017proximal}. In addition, we test GSPO \citep{zheng2025group}, which improves upon GRPO by defining importance ratios at the sequence level rather than the token level. Given our out-of-distribution setup, we also consider methods specifically designed to enhance exploration. KL-Cov \citep{cui2025entropy} applies a KL penalty to tokens to tokens with high covariance between their log probabilities and advantages, helping to prevent entropy collapse, while GRPO-Unlikeliness \citep{he2025rewarding} introduces an unlikeliness reward, which incentivizes correct but less probable solutions. 

\paragraph{Training details}
We train all models using the VERL framework \citep{sheng2024hybridflow}, starting from the Qwen-2.5-7B-Instruct model \citep{yang2025qwen3}. For the Omega tasks, we run training for 30 epochs with a batch size of 500 and a mini-batch size of 96. For Reasoning Gym, we train for 5 epochs, with a batch size of 64 and a mini-batch size of 32. We set PPO epochs to 1 for all experiments. The actor and critic use learning rates of $1\text{e}{-6}$ and $1\text{e}{-5}$, respectively. For the KL loss between the policy and base model, we use the following: 0.1 on the Probability (No Fixed) task, as well as on Reasoning Gym for the GRPO-Unlikeliness method, 0.05 for the other Omega tasks, and 0.01 for all other Reasoning Gym experiments. Across all methods, we use 32 (OMEGA) and 8 (Reasoning Gym) rollouts per prompt which, for GRPO-style methods, corresponds to the group size. We select the best performing checkpoint on the iid validation sets for evaluation. All experiments were performed on a single node with 8 NVIDIA H200 GPUs.

\subsection{Model analysis using the {\covtauunicode} curve}

The right side of Figure~\ref{fig:pass_at_k_example} illustrates the {\covtau} on the IID split of the Probability (no fixed) subset of OMEGA. Our metric highlights the brittleness of the base model: while it solves problems at very small thresholds, its coverage quickly decreases even at modest thresholds, such as $\tau$ = 0.2. KL-Cov consistently solves the most problems at almost all the reliability thresholds. 

In Figure \ref{fig:cover_tau_probs} we plot both the {\passk} and {\covtau} curves for the OOD split of the same dataset. All models have poor Pass@1 performance, but, for large sampling budgets, the base model significantly outperforms the RLVR models. For very small values of $\tau <0.01$, {\covtau} highlights a similarly strong performance of the base model. This correlates with the theoretical insights from Section \ref{theory}, showing that as $\tau$ approaches 0, it becomes similar to {\passk} at large $k$s.  However, when marginally increasing the threshold to $\tau = 0.025$, the performance of the base model drops significantly, which showcases a less optimistic view of its reasoning capabilities. Additionally, the {\covtau} curve reveals different trade-offs between the RLVR models. GAE solves more problems than GRPO-Unlikeliness for $\tau < 0.050$, while GRPO-Unlikeliness is more reliable for larger $\tau$ values. 

\paragraph{$\covtau$ offers finer granularity.}
By inspecting the entire curve, we can distinguish between algorithms that
(a) succeed rarely across many tasks, and
(b) succeed reliably on fewer tasks.
$\passk$ targets the former, while being uninformative about the latter. $\covtau$ exposes important trade-offs
between \emph{coverage} and \emph{reliability} of exploration.

\subsection{Results}

\begin{table}[t]
\caption{Results on the Intermediate Integration task from Reasoning Gym and the OMEGA benchmark (averaged across 5 tasks). Best results per metric are shown in bold, second and third best are highlighted in blue and orange, respectively.}
\label{omega_results}
\centering
\setlength{\tabcolsep}{2pt} %
\begin{tabular}{lcccccccc}
\hline
& \multicolumn{4}{c}{\textbf{Reasoning Gym}} & \multicolumn{4}{c}{\textbf{OMEGA OOD}} \\
\textbf{Method} & Pass@1 & cov@.2 & cov@.8 & $\mathrm{AvgAUC}^{+}_{\text{cov}}$ & Pass@1 & cov@.2 & cov@.8 & $\mathrm{AvgAUC}^{+}_{\text{cov}}$ \\
\hline
base     & 49.67  & 55.67 & 39.67 & 0.64 & 8.34 & 14.94 &  1.10 &  0.19 \\
GRPO         &\textbf{59.66} & \textbf{66.00} & \third{48.67}& \second{5.52} & 17.86 & \second{22.66} & 11.98 &  1.61 \\
GSPO         & 56.14& 57.67& 48.00&1.79 & \third{18.00} & 20.26 & \second{15.36} &  \second{2.36} \\
PPO (GAE)         &\third{57.94} &\third{60.00} & \second{55.67} & \third{5.13} & \second{18.38} & \second{22.66} & \third{13.68} & \third{1.85} \\
KL-Cov        & \second{58.55}& \second{62.00} & \textbf{56.33}& \textbf{5.70}& \textbf{28.34} & \textbf{33.58} & \textbf{23.34} & \textbf{12.78}  \\
Unlikeliness & 43.98& 51.67& 34.00& 0.00& 17.02 & \third{20.94} & 11.94 & 0.68 \\
\hline
\end{tabular}
\end{table}

In Table~\ref{omega_results} we evaluate several RLVR algorithms on the Reasoning Gym and OMEGA datasets. We report pass@1, as well as {\covtau} for $\tau=0.2$ and $\tau=0.8$, to assess both low and high reliability performance. While these point metrics provide insight into the performance at explicit reliability levels, they do not capture how models perform across the entire range of reliability thresholds. To summarize a model's performance across all $\tau$ regions, we build on the previously defined pairwise measure $\mathrm{AUC}^{+}_{\text{cover}}(A, B)$. Consider a set of models $\mathcal{M}$, with $|\mathcal{M}|=M$. For a model A, we average its pairwise excess AUC score against all the other models:
\begin{equation}
\mathrm{AvgAUC}^{+}_{\text{cover}}(A) \;=\; \frac{1}{M-1} \sum_{\substack{B \in \mathcal{M} \\ B \neq A}} 
    \mathrm{AUC}^{+}_{\text{cover}}(A,B)
\label{avgauc}
\end{equation}
$\mathrm{AvgAUC}^{+}_{\text{cover}}(A)$ quantifies how much A tends to dominate other models on the $\covtau$ curve.

In Table~\ref{omega_results} we observe that the rankings of the top methods differ between {\passone}, $\covtau$ and $\mathrm{AvgAUC}^{+}_{\text{cover}}$. While cover@0.2 yields the same rankings as {\passone}, with GRPO as the top method, cover@0.8 ranks KL-cov first, highlighting its strength in preserving depth at higher thresholds.
The $\mathrm{AvgAUC}^{+}_{\text{cover}}$ score (\ref{avgauc}) averages performance across dominant threshold regions: KL-Cov is ranked first, while GRPO ranks second.

KL-Cov obtains the best results on the OMEGA dataset on all metrics. While GRPO ranks 4th based on {\passone}, it ranks 2nd based on Cover@0.2 (tied with PPO), highlighting better coverage at low reliability.  Enforcing higher reliability ($\tau=0.8$) produces a different ordering compared to $\tau=0.2$, where breadth is emphasized. GRPO and GRPO-Unlikeliness show good {\covtau} performance at lower thresholds, but their coverage at $\tau=0.8$ drops out of the top 3. The $\mathrm{AvgAUC}^{+}_{\text{cover}}$ score (\ref{avgauc}) captures this trade-off: GRPO and GRPO-Unlikeliness lag behind GSPO and PPO, which are ranked 2nd and 3rd, respectively. 
Overall, the strong performance of KL-Cov suggests that RLVR methods that prevent entropy collapse show stronger generalization abilities.

\section{Limitations}
Our metric still accounts for the accuracy of the final answer, without evaluating the soundness of the reasoning trace, as was done in Cot-Pass@k\citep{wen2025reinforcement}. However, the accuracy of the reasoning trajectory can be combined with the {\covtau} metric.

\section{Conclusions}
We introduce $\covtau$ which emphasizes a \textit{reliability-controlled reasoning boundary}, by taking into account the ratio ${\tau}$ of correct completions for a given task. This fine-grained view of the performance can naturally highlight the trade-off between the coverage of solvable problems and the correct answer consistency. 

We connect {\passk} to {\covtau} and demonstrate that $\passk$ can be expressed as a weighted average of {\covtau}, that is biased towards low ${\tau}$ regions, making it prone to emphasizing lucky hits rather than reliability. Moreover, {\covtau} reveals ordering differences across reliability levels that {\passk} may hide. Through this new lens, we evaluate several RLVR methods to uncover their reasoning abilities under different reliability thresholds. 

\clearpage
\bibliography{iclr2026/main}

\begin{thebibliography}{24}
\providecommand{\natexlab}[1]{#1}
\providecommand{\url}[1]{\texttt{#1}}
\expandafter\ifx\csname urlstyle\endcsname\relax
  \providecommand{\doi}[1]{doi: #1}\else
  \providecommand{\doi}{doi: \begingroup \urlstyle{rm}\Url}\fi

\bibitem[Agarwal et~al.(2021)Agarwal, Schwarzer, Castro, Courville, and Bellemare]{agarwal2021deep}
Rishabh Agarwal, Max Schwarzer, Pablo~Samuel Castro, Aaron~C Courville, and Marc Bellemare.
\newblock Deep reinforcement learning at the edge of the statistical precipice.
\newblock \emph{Advances in neural information processing systems}, 34:\penalty0 29304--29320, 2021.

\bibitem[Chen et~al.(2021)Chen, Tworek, Jun, Yuan, Pinto, Kaplan, Edwards, Burda, Joseph, Brockman, et~al.]{chen2021evaluating}
Mark Chen, Jerry Tworek, Heewoo Jun, Qiming Yuan, Henrique Ponde De~Oliveira Pinto, Jared Kaplan, Harri Edwards, Yuri Burda, Nicholas Joseph, Greg Brockman, et~al.
\newblock Evaluating large language models trained on code.
\newblock \emph{arXiv preprint arXiv:2107.03374}, 2021.

\bibitem[Chen et~al.(2025)Chen, Qin, Wu, Ling, Ye, Zhao, and Shi]{chen2025pass}
Zhipeng Chen, Xiaobo Qin, Youbin Wu, Yue Ling, Qinghao Ye, Wayne~Xin Zhao, and Guang Shi.
\newblock Pass@ k training for adaptively balancing exploration and exploitation of large reasoning models.
\newblock \emph{arXiv preprint arXiv:2508.10751}, 2025.

\bibitem[Cheng et~al.(2025{\natexlab{a}})Cheng, Huang, Zhu, Dai, Zhao, Zhang, and Wei]{cheng2025reasoning}
Daixuan Cheng, Shaohan Huang, Xuekai Zhu, Bo~Dai, Wayne~Xin Zhao, Zhenliang Zhang, and Furu Wei.
\newblock Reasoning with exploration: An entropy perspective.
\newblock \emph{arXiv preprint arXiv:2506.14758}, 2025{\natexlab{a}}.

\bibitem[Cheng et~al.(2025{\natexlab{b}})Cheng, Hao, Liu, Zhou, Xie, Yao, Bian, Zhuang, Dey, Zha, et~al.]{cheng2025revisiting}
Zhoujun Cheng, Shibo Hao, Tianyang Liu, Fan Zhou, Yutao Xie, Feng Yao, Yuexin Bian, Yonghao Zhuang, Nilabjo Dey, Yuheng Zha, et~al.
\newblock Revisiting reinforcement learning for llm reasoning from a cross-domain perspective.
\newblock \emph{arXiv preprint arXiv:2506.14965}, 2025{\natexlab{b}}.

\bibitem[Cui et~al.(2025)Cui, Zhang, Chen, Yuan, Wang, Zuo, Li, Fan, Chen, Chen, et~al.]{cui2025entropy}
Ganqu Cui, Yuchen Zhang, Jiacheng Chen, Lifan Yuan, Zhi Wang, Yuxin Zuo, Haozhan Li, Yuchen Fan, Huayu Chen, Weize Chen, et~al.
\newblock The entropy mechanism of reinforcement learning for reasoning language models.
\newblock \emph{arXiv preprint arXiv:2505.22617}, 2025.

\bibitem[Guo et~al.(2025)Guo, Yang, Zhang, Song, Zhang, Xu, Zhu, Ma, Wang, Bi, et~al.]{guo2025deepseek}
Daya Guo, Dejian Yang, Haowei Zhang, Junxiao Song, Ruoyu Zhang, Runxin Xu, Qihao Zhu, Shirong Ma, Peiyi Wang, Xiao Bi, et~al.
\newblock Deepseek-r1: Incentivizing reasoning capability in llms via reinforcement learning.
\newblock \emph{arXiv preprint arXiv:2501.12948}, 2025.

\bibitem[He et~al.(2025)He, Fried, and Welleck]{he2025rewarding}
Andre He, Daniel Fried, and Sean Welleck.
\newblock Rewarding the unlikely: Lifting grpo beyond distribution sharpening.
\newblock \emph{arXiv preprint arXiv:2506.02355}, 2025.

\bibitem[Liu et~al.(2024)Liu, Liu, Xiao, Wang, Liu, Gao, Zhang, Zhang, and Chen]{liu2024your}
Junnan Liu, Hongwei Liu, Linchen Xiao, Ziyi Wang, Kuikun Liu, Songyang Gao, Wenwei Zhang, Songyang Zhang, and Kai Chen.
\newblock Are your llms capable of stable reasoning?
\newblock \emph{arXiv preprint arXiv:2412.13147}, 2024.

\bibitem[Liu et~al.(2025{\natexlab{a}})Liu, Diao, Lu, Hu, Dong, Choi, Kautz, and Dong]{liu2025prorl}
Mingjie Liu, Shizhe Diao, Ximing Lu, Jian Hu, Xin Dong, Yejin Choi, Jan Kautz, and Yi~Dong.
\newblock Prorl: Prolonged reinforcement learning expands reasoning boundaries in large language models.
\newblock \emph{arXiv preprint arXiv:2505.24864}, 2025{\natexlab{a}}.

\bibitem[Liu et~al.(2025{\natexlab{b}})Liu, Chen, Li, Qi, Pang, Du, Lee, and Lin]{liu2025understanding}
Zichen Liu, Changyu Chen, Wenjun Li, Penghui Qi, Tianyu Pang, Chao Du, Wee~Sun Lee, and Min Lin.
\newblock Understanding r1-zero-like training: A critical perspective.
\newblock \emph{arXiv preprint arXiv:2503.20783}, 2025{\natexlab{b}}.

\bibitem[Schulman et~al.(2017)Schulman, Wolski, Dhariwal, Radford, and Klimov]{schulman2017proximal}
John Schulman, Filip Wolski, Prafulla Dhariwal, Alec Radford, and Oleg Klimov.
\newblock Proximal policy optimization algorithms.
\newblock \emph{arXiv preprint arXiv:1707.06347}, 2017.

\bibitem[Shao et~al.(2024)Shao, Wang, Zhu, Xu, Song, Bi, Zhang, Zhang, Li, Wu, et~al.]{shao2024deepseekmath}
Zhihong Shao, Peiyi Wang, Qihao Zhu, Runxin Xu, Junxiao Song, Xiao Bi, Haowei Zhang, Mingchuan Zhang, YK~Li, Yang Wu, et~al.
\newblock Deepseekmath: Pushing the limits of mathematical reasoning in open language models.
\newblock \emph{arXiv preprint arXiv:2402.03300}, 2024.

\bibitem[Sheng et~al.(2024)Sheng, Zhang, Ye, Wu, Zhang, Zhang, Peng, Lin, and Wu]{sheng2024hybridflow}
Guangming Sheng, Chi Zhang, Zilingfeng Ye, Xibin Wu, Wang Zhang, Ru~Zhang, Yanghua Peng, Haibin Lin, and Chuan Wu.
\newblock Hybridflow: A flexible and efficient rlhf framework.
\newblock \emph{arXiv preprint arXiv: 2409.19256}, 2024.

\bibitem[Stojanovski et~al.(2025)Stojanovski, Stanley, Sharratt, Jones, Adefioye, Kaddour, and K{\"o}pf]{stojanovski2025reasoning}
Zafir Stojanovski, Oliver Stanley, Joe Sharratt, Richard Jones, Abdulhakeem Adefioye, Jean Kaddour, and Andreas K{\"o}pf.
\newblock Reasoning gym: Reasoning environments for reinforcement learning with verifiable rewards.
\newblock \emph{arXiv preprint arXiv:2505.24760}, 2025.

\bibitem[Sun et~al.(2025)Sun, Hu, Zhou, Zheng, Hajishirzi, Dziri, and Song]{sun2025omega}
Yiyou Sun, Shawn Hu, Georgia Zhou, Ken Zheng, Hannaneh Hajishirzi, Nouha Dziri, and Dawn Song.
\newblock Omega: Can llms reason outside the box in math? evaluating exploratory, compositional, and transformative generalization.
\newblock \emph{arXiv preprint arXiv:2506.18880}, 2025.

\bibitem[Wang et~al.(2025{\natexlab{a}})Wang, Yu, Gao, Zheng, Liu, Lu, Dang, Chen, Yang, Zhang, et~al.]{wang2025beyond}
Shenzhi Wang, Le~Yu, Chang Gao, Chujie Zheng, Shixuan Liu, Rui Lu, Kai Dang, Xionghui Chen, Jianxin Yang, Zhenru Zhang, et~al.
\newblock Beyond the 80/20 rule: High-entropy minority tokens drive effective reinforcement learning for llm reasoning.
\newblock \emph{arXiv preprint arXiv:2506.01939}, 2025{\natexlab{a}}.

\bibitem[Wang et~al.(2025{\natexlab{b}})Wang, Yang, Zeng, Ren, Liu, Peng, Cheng, He, Wang, Gao, et~al.]{wang2025reinforcement}
Yiping Wang, Qing Yang, Zhiyuan Zeng, Liliang Ren, Liyuan Liu, Baolin Peng, Hao Cheng, Xuehai He, Kuan Wang, Jianfeng Gao, et~al.
\newblock Reinforcement learning for reasoning in large language models with one training example.
\newblock \emph{arXiv preprint arXiv:2504.20571}, 2025{\natexlab{b}}.

\bibitem[Wen et~al.(2025)Wen, Liu, Zheng, Xu, Ye, Wu, Liang, Wang, Li, Miao, et~al.]{wen2025reinforcement}
Xumeng Wen, Zihan Liu, Shun Zheng, Zhijian Xu, Shengyu Ye, Zhirong Wu, Xiao Liang, Yang Wang, Junjie Li, Ziming Miao, et~al.
\newblock Reinforcement learning with verifiable rewards implicitly incentivizes correct reasoning in base llms.
\newblock \emph{arXiv preprint arXiv:2506.14245}, 2025.

\bibitem[Wu et~al.(2025{\natexlab{a}})Wu, Xuan, Lu, Harchaoui, and Choi]{wu2025invisible}
Fang Wu, Weihao Xuan, Ximing Lu, Zaid Harchaoui, and Yejin Choi.
\newblock The invisible leash: Why rlvr may not escape its origin.
\newblock \emph{arXiv preprint arXiv:2507.14843}, 2025{\natexlab{a}}.

\bibitem[Wu et~al.(2025{\natexlab{b}})Wu, Zhang, Dong, Xi, Zhao, Jin, Fan, Zhou, Lv, Zhang, et~al.]{wu2025reasoning}
Mingqi Wu, Zhihao Zhang, Qiaole Dong, Zhiheng Xi, Jun Zhao, Senjie Jin, Xiaoran Fan, Yuhao Zhou, Huijie Lv, Ming Zhang, et~al.
\newblock Reasoning or memorization? unreliable results of reinforcement learning due to data contamination.
\newblock \emph{arXiv preprint arXiv:2507.10532}, 2025{\natexlab{b}}.

\bibitem[Yang et~al.(2025)Yang, Li, Yang, Zhang, Hui, Zheng, Yu, Gao, Huang, Lv, et~al.]{yang2025qwen3}
An~Yang, Anfeng Li, Baosong Yang, Beichen Zhang, Binyuan Hui, Bo~Zheng, Bowen Yu, Chang Gao, Chengen Huang, Chenxu Lv, et~al.
\newblock Qwen3 technical report.
\newblock \emph{arXiv preprint arXiv:2505.09388}, 2025.

\bibitem[Yu et~al.(2025)Yu, Zhang, Zhu, Yuan, Zuo, Yue, Dai, Fan, Liu, Liu, et~al.]{yu2025dapo}
Qiying Yu, Zheng Zhang, Ruofei Zhu, Yufeng Yuan, Xiaochen Zuo, Yu~Yue, Weinan Dai, Tiantian Fan, Gaohong Liu, Lingjun Liu, et~al.
\newblock Dapo: An open-source llm reinforcement learning system at scale.
\newblock \emph{arXiv preprint arXiv:2503.14476}, 2025.

\bibitem[Zheng et~al.(2025)Zheng, Liu, Li, Chen, Yu, Gao, Dang, Liu, Men, Yang, et~al.]{zheng2025group}
Chujie Zheng, Shixuan Liu, Mingze Li, Xiong-Hui Chen, Bowen Yu, Chang Gao, Kai Dang, Yuqiong Liu, Rui Men, An~Yang, et~al.
\newblock Group sequence policy optimization.
\newblock \emph{arXiv preprint arXiv:2507.18071}, 2025.

\end{thebibliography}
\bibliographystyle{iclr2026_conference}


\end{document}